\documentclass[letterpaper, 10 pt, conference]{ieeeconf}  

\usepackage{lastpage}
\usepackage{fontawesome}
\def\pscircled#1{\textcircled{\resizebox{.5em}{!}{#1}}}
\usepackage{cite}  % Only once
\IEEEoverridecommandlockouts                              
\usepackage{graphics} % for pdf, bitmapped graphics files
\usepackage{epsfig}   % for postscript graphics files
\usepackage{mathptmx} % assumes new font selection scheme installed
\usepackage{times}    % assumes new font selection scheme installed
\usepackage{amsmath}  % assumes amsmath package installed
\usepackage{amssymb}  % assumes amsmath package installed
\usepackage[utf8]{inputenc}
\usepackage[english]{babel}
\newtheorem{assumption}{Assumption}
\usepackage{cite}
\usepackage{float}
\usepackage{amsfonts}
\usepackage{booktabs}
\usepackage{epsfig}
\usepackage{epstopdf}
\usepackage{subfigure}
\usepackage{cite}
\usepackage{multirow} % multirow for format of table
\usepackage{algorithm,algpseudocode}
\usepackage{subcaption}
\usepackage{xcolor}

\usepackage[linkcolor=blue,colorlinks=true]{hyperref}

\newtheorem{definition}{Definition}
\newtheorem{theorem}{Theorem}
\newtheorem{lemma}[theorem]{Lemma}

\newtheorem{proposition}{Proposition}[definition]

\usepackage{caption}
\renewcommand{\baselinestretch}{1.0}

\title{\LARGE \bf
 Effective Dimension Aware  Fractional-Order Stochastic Gradient Descent for Convex Optimization Problems
}

\author{Mohammad Partohaghighi$^{1}$, Roummel Marcia$^{2}$, and YangQuan Chen$^{1,3}$% <-this % stops a space
\thanks{*Corresponding author: YangQuan Chen.  \pscircled{\faPhone}~{\tt +1-209-2284672}.
YC is supported by the Center for Methane Emission Research and Innovation (\href{http://methane.ucmerced.edu}{\tt CMERI}) through the Climate Action Seed Funds grant (2023-2026) at University of California, Merced. \pscircled{\faDesktop}~\href{mechatronics.ucmerced.edu}{\tt mechatronics.ucmerced.edu} }% <-this % stops a space
\thanks{$^{1}$Mohammad Partohaghighi and YangQuan Chen are with the Electrical Engineering and Computer Science graduate program, 
        University of California at Merced, USA 
          \pscircled{\faEnvelope}~{\tt\small mpartohaghighi@ucmerced.edu}}%
\thanks{$^{2}$Roummel Marcia is with the Department of Applied Mathematics, 
        University of California Merced, Merced, CA, USA 
          \pscircled{\faEnvelope}~{\tt\small rmarcia@ucmerced.edu}}%
\thanks{$^{3}$YangQuan Chen is with the Mechatronics, Embedded Systems and Automation (MESA) Lab,         Department of Mechanical Engineering, School of Engineering, 
        University of California, Merced, CA 95343, USA 
          \pscircled{\faEnvelope}~{\tt\small ychen53@ucmerced.edu}}%
}

\begin{document}

\maketitle
%\thispagestyle{empty}
%\pagestyle{empty}

%\pagenumbering{arabic}

%%%%%%%%%%%%%%%%%%%%%%%%%%%%%%%%%%%%%%%%%%%%%%%%%%%%%%%%%%%%%%%%%%%%%%%%%%%%%%%%
\begin{abstract}

Fractional-order stochastic gradient descent (FOSGD) leverages fractional exponents to capture long-memory effects in optimization. However, its utility is often limited by the difficulty of tuning and stabilizing these exponents. We propose 2SED Fractional-Order Stochastic Gradient Descent (2SEDFOSGD), which integrates the Two-Scale Effective Dimension (2SED) algorithm with FOSGD to adapt the fractional exponent in a data-driven manner. By tracking model sensitivity and effective dimensionality, 2SEDFOSGD dynamically modulates the exponent to mitigate oscillations and hasten convergence. Theoretically, this approach preserves the advantages of fractional memory without the sluggish or unstable behavior observed in na\"ive fractional SGD. Empirical evaluations in Gaussian and $\alpha$-stable noise scenarios using an autoregressive (AR) model\textcolor{red}{, as well as on the MNIST and CIFAR-100 datasets for image classification,} highlight faster convergence and more robust parameter estimates compared to baseline methods, underscoring the potential of dimension-aware fractional techniques for advanced modeling and estimation tasks.
\end{abstract}


%%%%%%%%%%%%%%%%%%%%%%%%%%%%%%%%%%%%%%%%%%%%%%%%%%%%%%%%%%%%%%%%%%%%%%%%%%%%%%%%
%%% Introduction
% \linespread{0.994} % Adjust the value as needed (1.0 is default)
\renewcommand{\baselinestretch}{.94}

\section{Introduction}

Machine learning (ML) and scientific computing increasingly rely on sophisticated optimization methods to tackle complex, high-dimensional problems. Classical stochastic gradient descent (SGD) has become a mainstay in training neural networks and large-scale models, owing to its simplicity and practical performance. However, standard SGD exhibits notable limitations: it typically treats updates as short-term corrections, discarding a rich history of past gradients. In contrast, fractional approaches in optimization draw upon the theory of fractional calculus to capture long-memory effects, thereby influencing the trajectory of updates by retaining historical gradient information over extended intervals \cite{zhou2025fractional, sheng2020fractional}.

\textcolor{red}{We assume the parameter space \(\Theta \subseteq \mathbb{R}^d\) represents the parameters of a neural network with \( L \) layers, where each layer \(j\) has parameters \(\theta^{(j)} \in \mathbb{R}^{d_j}\), and \( d = \sum_{j=1}^L d_j \). The Fisher Information Matrix and 2SED measures are computed for each layer to adapt optimization updates.}

Fractional calculus extends traditional calculus to include non-integer orders, offering a powerful tool for modeling and control in various fields, including optimization. It allows for the incorporation of memory and hereditary properties into models, which is particularly beneficial in dynamic systems and optimization \cite{Sheng2012} and\textcolor{red}{\cite{Yu2022FractionalOrderMomentum}}. By leveraging fractional derivatives, optimization algorithms can potentially achieve better convergence properties and robustness against noise, as they account for the accumulated effect of past gradients rather than relying solely on the most recent updates \cite{chen2018fractional2, liu2021quasi}. This approach has shown promise in enhancing the performance of optimization algorithms in machine learning and other scientific computing applications \cite{chen2017study}.

 By embracing these generalized derivatives, FOSGD modifies the usual gradient step to incorporate a partial summation of past gradients, effectively smoothing updates over a historical window. The method stands especially valuable for scenarios where prior states wield significant impact on the current gradient, as often encountered in dynamic processes or highly non-convex landscape \cite{shin2023accelerating}.
Nonetheless, the quest for harnessing fractional updates is not without drawbacks. Incorporating fractional operators demands added hyperparameters (particularly the fractional exponent $\alpha$), which can prove sensitive or unstable to tune. Excessively low or high fractional orders may slow convergence or lead to oscillatory gradients, thus negating the presumed benefits. Bridging the gap between the theoretical elegance of fractional calculus and the pressing computational demands of real-world ML systems remains a formidable challenge.

Although FOSGD helps mitigate short-term memory loss by preserving traces of past gradients, selecting and calibrating the fractional exponents can introduce substantial complexities in real-world settings. For instance, deciding whether \(\alpha = 0.5\) or \(\alpha = 0.9\) is most appropriate for capturing relevant memory structures is neither straightforward nor reliably robust, with the optimal choice often varying considerably across tasks or even across different stages of training. In practice, if the chosen exponent fails to align with the true dynamics of the loss landscape, updates may drift or stall, resulting in unpredictable or sluggish convergence. Moreover, fractional terms can amplify variance in gradient estimatesespecially under noisy or non-stationary conditions thereby causing oscillatory or chaotic training behaviors that undermine stability.

Beyond these convergence and stability concerns, fractional exponents impose additional burdens on tuning and hyperparameter selection. Even minor changes in \(\alpha\) can radically alter the memory effect, forcing practitioners to engage in extensive trial-and-error experiments to achieve consistent results. Such overhead becomes especially prohibitive in large-scale or time-sensitive applications, where iterating over a range of fractional parameters is not feasible. Consequently, despite its theoretical promise as a memory-based learning strategy, FOSGD faces limited adoption in practice, as the algorithm's strong reliance on well-chosen exponents can undercut the potential advantages that long-range gradient retention might otherwise provide. 

\textcolor{red}{Studies have highlighted issues such as the need for precise tuning of fractional orders to avoid erratic convergence paths \cite{elnady2024comprehensive, harjule2024fractional}. Additionally, challenge of converging to a real extreme point encountered by the existing fractional gradient algorithms is addressed in \cite{chen2024lambda, herrera2022fractional,chen2024novel}. These challenges underscore the need for robust fractional SGD variants that balance computational efficiency with stable convergence.}
A geometry-aware strategy like Two-Scale Effective Dimension (2SED) can dynamically regulate the fractional exponent in FOSGD. By examining partial diagonal approximations of the Fisher information matrix \cite{datres2024two}, 2SED identifies regions of high sensitivity and adapts the exponent accordingly. This approach dampens updates in areas prone to instability while exploiting longer memory in flatter regions. Consequently, combining 2SED with FOSGD reduces erratic oscillations, preserves long-term memory benefits, and yields more robust performance across diverse data sets and problem types. We introduce a novel 2SED-driven FOSGD framework that dynamically regulates the fractional exponent using the dimension-aware metrics of 2SED. This adaptive mechanism aligns historical-gradient memory with the sensitivity of the optimization landscape, thereby enhancing stability and data alignment. Under standard smoothness and bounded-gradient assumptions, the method satisfies strong convergence criteria. In practice, geometry-based regularization fosters a more consistent convergence, as evidenced by solving an autoregressive (AR) model under  Gaussian and $\alpha$-stable noise. We organize this paper as follows. In Section~\ref{sec:2sed_frac_sgd}, we thoroughly examine the 2SED algorithm, illustrating how it approximates second-order geometry to produce dimension-aware updates. Section~\ref{sec:frac_sgd_highlvl} reviews FOSGD, highlighting its appeal for long-memory processes and the hyperparameter dilemmas that hinder practicality. Also, we detail how to embed 2SED’s dimension metrics into the fractional framework, providing both equations and pseudo-code. Section~\ref{sec:convergence_convex} delves into a convergence analysis, establishing theoretical performance bounds for our method. Section~\ref{sec:experiments} showcases experiments \textcolor{red}{across different tasks, such as an auto-regressive (AR) model and image classification, demonstrating that} 2SED-driven exponent adaptation yields measurably stronger results. Finally, Section~\ref{sec:conclusion} concludes by summarizing key findings, identifying broader implications for optimization, and suggesting directions for further research in advanced fractional calculus and dimension-based learning techniques. This overarching narrative underscores the growing intersection between fractional approaches and dimension-aware strategies. Aligning memory-based methods with geometry-aware design, we move closer to an optimization paradigm that capitalizes on historical information without succumbing to the pitfalls of unbounded memory effects. Our findings thus underscore the promise of \emph{2SED + FOSGD} as a more stable algorithmic solution poised for wide adoption in deep learning.
%%%%%%%%%%%%%%%%%%%%%%%%%%%%%%%%%%%%%%%%%%%%%%%%%%%%%%%%%%%%%%%%%%%%%%%%%%%%%%%%%%%%%%%%%%%%%%%%%%%%%%%%%%%%%%%%%%%%%%%%%%%%%%%%%%%%%%%%%%%%%%%%%%%%%%%%%%%%%%%%%%%%%%%%%%%%%%%%%%%%%%%%%%%%%%%%%%%

\section{Two-Scale Effective Dimension (2SED) and Fractional-Order SGD}
\label{sec:2sed_frac_sgd}

Classical complexity measures, such as the Vapnik-Chervonenkis (VC) dimension \cite{vapnik1999nature} or raw parameter counts, often overestimate the capacity of overparameterized neural networks. Zhang et al. \cite{zhang2017understanding} demonstrate that deep networks, such as Inception-style models with millions of parameters, generalize well despite their ability to memorize random labels, undermining naive VC-based bounds. This discrepancy arises because many directions in the high-dimensional parameter space \(\Theta \subseteq \mathbb{R}^d\) are ``flat,'' contributing minimally to model outputs, while a subset of sensitive directions dominates learning \cite{karakida2019universal}.

Curvature-aware approaches, leveraging the Fisher Information Matrix (FIM) \cite{amari1998natural}, better capture local sensitivity. We adopt the Two-Scale Effective Dimension (2SED) \cite{datres2024two}, which integrates global parameter counts with local curvature effects encoded in the FIM, offering a more nuanced complexity measure than Hessian-based metrics \cite{liang2019fisher} or K-FAC approximations \cite{martens2015optimizing}. In this section, we define 2SED and its layer-wise variant, Lower 2SED, and propose their use in adapting fractional-order stochastic gradient descent (FOSGD) to improve optimization stability and generalization.

\subsection{Foundational Definitions}
We consider a neural network with \( L \) layers, where layer \( j \) has parameters \(\theta_j \in \mathbb{R}^{d_j}\), and the parameter space is \(\Theta = \Theta_1 \times \cdots \times \Theta_L \subseteq \mathbb{R}^d\), with \( d = \sum_{j=1}^L d_j \). The Fisher Information Matrix (FIM) and 2SED are computed layer-wise to adapt optimization updates, leveraging the Markovian structure of feed-forward networks \cite{datres2024two}. Each layer’s parameter vector \(\theta_j \in \Theta_j \subseteq \mathbb{R}^{d_j}\) represents the trainable weights and biases, while \(\Theta_j\) is the bounded domain of possible parameter values, ensuring regularity in the statistical model.

\begin{definition}[Fisher Information \cite{datres2024two}]
\label{def:fisher-info}
For a statistical model \( p_{\theta}(x, y) \) with parameters \( \theta \in \Theta \subseteq \mathbb{R}^d \), assuming \( p_{\theta} \) is differentiable and non-degenerate, define the log-likelihood as
\[
\ell_{\theta}(x, y) = \log p_{\theta}(x, y).
\]
The \emph{Fisher Information Matrix} \( F(\theta) \) is given by
\begin{equation}
\label{eq:FisherInfo}
F(\theta) = \mathbb{E}_{(x, y) \sim p_{\theta}} \left[ \left( \nabla_{\theta} \ell_{\theta}(x, y) \right) \otimes \left( \nabla_{\theta} \ell_{\theta}(x, y) \right) \right],
\end{equation}
where \( \otimes \) denotes the outer product and the expectation is over \( p_{\theta} \). Under regularity conditions, this equals \( \mathbb{E}[-\nabla_{\theta}^2 \ell_{\theta}(x, y)] \) \cite{amari1998natural}.
\end{definition}

\begin{definition}[Empirical Fisher ]\cite{datres2024two}
\label{def:empirical-fisher}
Given an i.i.d.\ sample \( \{(X_i, Y_i)\}_{i=1}^N \), the \emph{empirical Fisher Information Matrix} is
\begin{equation}
\label{eq:EmpiricalFisher}
F_N(\theta) = \frac{1}{N} \sum_{i=1}^N \left( \nabla_{\theta} \ell_{\theta}(X_i, Y_i) \right) \otimes \left( \nabla_{\theta} \ell_{\theta}(X_i, Y_i) \right),
\end{equation}
converging to \( F(\theta) \) as \( N \to \infty \).
\end{definition}

\begin{definition}[Normalized Fisher Matrix \cite{datres2024two}]
\label{def:normalized-Fisher}
The \emph{normalized Fisher matrix} \( \widehat{F}(\theta) \) rescales \( F(\theta) \) so that
\[
\mathbb{E}_{\theta}[\mathrm{Tr} \widehat{F}(\theta)] = d,
\]
where \( d = \dim(\Theta) \). Formally,
\begin{equation}
\label{eq:NormalizedF}
\widehat{F}(\theta) =
\begin{cases}
\frac{d}{\mathbb{E}_{\theta} [\mathrm{Tr} F(\theta)]} F(\theta), & \text{if } \mathbb{E}_{\theta}[\mathrm{Tr} F(\theta)] > 0, \\
0, & \text{otherwise}.
\end{cases}
\end{equation}
\end{definition}

\subsection{\textcolor{red}{The 2SED Approach}}
\label{subsec:2sed}
\textcolor{red}{Although \( d = \dim(\Theta) \) represents the nominal number of parameters, many directions in \(\Theta\) are flat, contributing minimally to the loss \cite{karakida2019universal}. The Two-Scale Effective Dimension (2SED) integrates a curvature-based term derived from the Fisher Information Matrix with the parameter count \( d \), capturing the effective dimensionality of active directions.}
\begin{definition}[Two-Scale Effective Dimension \cite{datres2024two}]
\label{def:2sed}
Let \( \widehat{F}(\theta) \) be the normalized Fisher matrix, positive semi-definite under mild conditions. For \( 0 < \varepsilon < 1 \) and \( \zeta \in \left[ \frac{2}{3}, 1 \right) \), the 2SED is:
\begin{equation}
\label{eq:2SED}
d_{\zeta}(\varepsilon) = \zeta d + (1 - \zeta) d_{\text{curv}}(\varepsilon),
\end{equation}
where
\begin{equation}
\label{eq:dcurv}
d_{\text{curv}}(\varepsilon) = \frac{\log \mathbb{E}_{\theta} \left[ \det \left( I_d + \varepsilon^{\zeta-1} \widehat{F}(\theta)^{\frac{1}{2}} \right) \right]}{\bigl|\log \left( \varepsilon^{\zeta-1} \right)\bigr|}.
\end{equation}
Here, \( I_d \) is the \( d \times d \) identity matrix, and \( \widehat{F}(\theta)^{\frac{1}{2}} \) is the positive semi-definite square root of \( \widehat{F}(\theta) \).
\end{definition}
\textcolor{red}{For layer-wise optimization, we compute \( d_{\zeta}^{(j)}(\varepsilon) \) for each layer \( j \), using the layer-wise FIM \( F_j(\theta_j) \), approximated empirically as in Definition \ref{def:empirical-fisher}. The parameter \(\zeta\) balances the nominal dimension \( d_j \) and the curvature term \( d_{\text{curv}}(\varepsilon) \). Smaller \(\varepsilon\) amplifies the contribution of significant eigenvalues, emphasizing high-curvature directions. As \(\zeta \to 0\), 2SED prioritizes curvature-based modes, while \(\zeta \to 1\) recovers the nominal dimension \( d_j \). The term \(\log \det \left( I_d + \varepsilon^{\zeta-1} \widehat{F}(\theta)^{\frac{1}{2}} \right)\) summarizes the spectrum of \( \widehat{F}(\theta)^{\frac{1}{2}} \), emphasizing directions with large eigenvalues (high curvature) while suppressing flat directions. Rewriting the determinant as \( \prod_i \left( 1 + \varepsilon^{\zeta-1} \lambda_i^{1/2} \right) \), we get:}
\begin{equation}
\label{eq:dcurv_diag}
d_{\text{curv}}(\varepsilon) = \frac{\sum_{i=1}^d \log \left( 1 + \varepsilon^{\zeta-1} \lambda_i^{1/2} \right)}{\bigl|\log \left( \varepsilon^{\zeta-1} \right)\bigr|},
\end{equation}
where \( \lambda_i \) are the eigenvalues of \( \widehat{F}(\theta) \). This aligns with information geometry, capturing the effective degrees of freedom in parameter space \cite{amari1998natural}.

\subsubsection{\textcolor{red}{Lower 2SED for Layer-wise Complexity}}
\label{subsec:lower_2sed}
\textcolor{red}{The Lower 2SED is a critical component of our 2SEDFOSGD algorithm, enabling efficient and adaptive optimization in deep neural networks. Unlike the global 2SED, which requires computing the Fisher Information Matrix (FIM) for all model parameters—a computationally prohibitive task for deep architectures like ResNet-50 with millions of parameters Lower 2SED leverages the Markovian structure of feed-forward networks to compute layer-wise complexity measures. This reduces memory requirements from \(O(d^2)\) to \(O(d_j^2)\) per layer, where \(d_j\) is the number of parameters in layer \(j\), and enables scalable computation. By providing a per-layer complexity metric, \(\underline{d}_\zeta^j(\varepsilon)\), Lower 2SED allows 2SEDFOSGD to dynamically adjust the fractional-order exponent \(\alpha_t^{(j)}\) for each layer, tailoring updates to the local curvature of the loss landscape. This leads to faster convergence and improved generalization compared to standard FOSGD, which uses a uniform fractional order \cite{Yang2023}. The Lower 2SED, introduced by Datres et al. \cite{datres2024two} for Markovian models like CNNs, is defined iteratively for each layer \( j = 1, \ldots, L \) of a model with parameters \( \theta = (\theta_1, \ldots, \theta_L) \), where \(\theta_j \in \Theta_j\). The FIM for layer \( j \) is:}
\begin{equation}
\begin{split}
F_j(\theta_1, \ldots, \theta_j) &= \mathbb{E}_{x_0, \ldots, x_{j-1}} \left[ \int_{\mathcal{X}_j} \left( \nabla_{\theta_j} \log p_{\theta_j}(x_j \mid x_{j-1}) \right) \right. \\
&\quad \left. \times \left( \nabla_{\theta_j} \log p_{\theta_j}(x_j \mid x_{j-1}) \right)^T p_{\theta_j}(dx_j \mid x_{j-1}) \right],
\end{split}
\end{equation}
\textcolor{red}{where \( p_{\theta_j}(x_j \mid x_{j-1}) \) is the conditional output distribution of layer \( j \). To model deterministic CNN outputs probabilistically, we assume layer outputs follow a Gaussian distribution with mean equal to the deterministic output and variance \(\sigma^2 = 0.01\), as in \cite{datres2024two}. The Lower 2SED is computed as:}
\begin{equation}
\begin{split}
\underline{d}_\zeta^j(\varepsilon) &= \underline{d}_\zeta^{j-1}(\varepsilon) \\
&+ \frac{1-\zeta}{|\log \varepsilon|} \oint_{\hat{\Theta}_j} \int_{\Theta_j} \log \det \left( I_j + \varepsilon^{\zeta-1} F_j(\theta_1, \ldots, \theta_j)^{\frac{1}{2}} \right) \\
&\quad \times d\theta_j d\Phi_j,
\end{split}
\end{equation}
\textcolor{red}{where \(\hat{\Theta}_j = \Theta_1 \times \cdots \times \Theta_{j-1}\), \( d\Phi_j \) is a normalized measure over previous layers’ parameters, and \( \underline{d}_\zeta^1(\varepsilon) \) is computed for the first layer (see \cite{datres2024two} for details). In practice, \( F_j \) is approximated empirically using Monte Carlo integration. The Lower 2SED, \(\underline{d}_\zeta^j(\varepsilon)\), replaces 2SED in our 2SEDFOSGD algorithm (Algorithm \ref{alg:2sed_frac_sgd}), scaling fractional-order gradients layer-wise to enhance convergence and generalization. The selection of the parameter \(\zeta\) in the two-scale effective dimension (2SED) is pivotal for balancing theoretical rigor and practical applicability in deep learning model complexity analysis. As specified in Theorem 5.1 in \cite{datres2024two}, \(\zeta \in \left[\frac{2}{3}, 1\right)\) ensures the validity of the generalization bound.}

\section{Fractional-Order SGD and 2SED Adaptation}
\label{sec:frac_sgd_highlvl}
Classical stochastic gradient descent (SGD) updates parameters using instantaneous gradients. However, optimization in deep learning often exhibits memory effects, suggesting that incorporating past gradients could improve convergence. Fractional calculus, via the Caputo derivative, provides a principled way to encode gradient history, with the fractional order \(\alpha\) controlling the memory effect \cite{monje2010fractional}.

\subsection{Caputo Fractional Derivative and Fractional Updates}
\begin{definition}[Caputo Derivative \cite{monje2010fractional}]
\label{def:Caputo}
For \( n-1 < \alpha < n \), the \emph{Caputo fractional derivative} of a function \( f \) is:
\begin{equation}
D_t^{\alpha} f(t) = \frac{1}{\Gamma(n - \alpha)} \int_0^t \frac{f^{(n)}(\tau)}{(t - \tau)^{\alpha - n + 1}} \, d\tau.
\end{equation}
The Caputo form is preferred in optimization as it handles initial conditions naturally and yields zero for constant functions \cite{Yang2023}.
\end{definition}
In discrete optimization, the classical gradient \(\nabla f(\theta_t)\) is replaced by the fractional gradient \( D_t^\alpha f(\theta_t) \). The fractional-order SGD update is \cite{Yang2023} $\theta_{t+1} = \theta_t - \eta D_t^{\alpha} f(\theta_t).$ For \(\alpha \in (0,1)\), \(\delta > 0\), and using a Taylor series approximation \cite{Yang2023}:
\begin{equation}
\theta_{t+2} = \theta_{t+1} - \mu_t \frac{\nabla f(\theta_{t+1})}{\Gamma(2 - \alpha)} \left( \lvert \theta_{t+1} - \theta_t \rvert + \delta \right)^{1 - \alpha}.
\end{equation}\label{eq:frac_gd_update_expanded}
The offset \(\delta > 0\) prevents stalling when consecutive iterates are similar.

\subsection{Adapting the Fractional Exponent via Lower 2SED}
\label{sec:modified_frac_sgd_2sed}
\textcolor{red}{
To adapt the fractional exponent for each layer $j$ of the neural network, where layer $j$ has parameters $\theta^{(j)} \in \mathbb{R}^{d_j}$, we compute the 2SED $ d_{\zeta}^{(j)}(\varepsilon)|_t $ for layer $j$ at iteration $ t $. A fixed fractional exponent $\alpha$ can lead to instability if the model's curvature changes dramatically during training. Intuitively, high curvature or high 2SED indicates directions of rapid change or ``sensitivity,'' so a smaller $\alpha_t^{(j)}$ (closer to 0) is preferred, as it increases the memory effect by amplifying the fractional term $\left( |\theta_{t+1}^{(j)} - \theta_t^{(j)}| + \delta \right)^{1-\alpha_t^{(j)}}$ in the update rule. This smooths updates, enhancing stability and preventing overshooting in these sensitive directions. Conversely, in regions with low curvature or low 2SED, which correspond to flatter areas of the optimization landscape, a larger $\alpha_t^{(j)}$ (closer to $\alpha_0$) is preferred, as it reduces the influence of the fractional term, making the update resemble standard SGD. This allows for faster convergence by relying more on the current gradient in regions where large steps are safer. Hence, we propose a 2SED-based FOSGD that dynamically adjusts $\alpha$ using 2SED of each layer. Suppose we compute the 2SED, $d_{\zeta}^{(j)}(\varepsilon)$, for layer $j$ and let $\alpha_t^{(j)} \;=\; \alpha_{0} \;-\; \beta \times \frac{ d_{\zeta}^{(j)}(\varepsilon)\big\rvert_{t} }{ d_{\max} },$ where $\alpha_0$ is a base fractional order, $\beta>0$ is a tuning parameter and $d_{\max}$ is the maximum observed 2SED among all layers. The fraction $\frac{d_{\zeta}^{(j)}(\varepsilon)\big\rvert_{t}}{d_{\max}}$ scales the current 2SED to the range $[0,1]$.
}
\subsection{2SEDFOSGD Algorithm}
\begin{algorithm}[H]
\caption{\textcolor{red}{2SED-Based Fractional-Order SGD (2SEDFOSGD)}}
\label{alg:2sed_frac_sgd}
\textbf{Input:} Neural network with \( L \) layers; parameters \(\theta^0 \in \mathbb{R}^d\); loss function \( f(\theta) \); base fractional order \(\alpha_0 \in (0,1]\); tuning parameter \(\beta > 0\); singularity offset \(\delta > 0\); base learning rate \(\mu_0 > 0\); 2SED balance parameter \(\zeta\); curvature sensitivity \(\varepsilon\); maximum iterations \( t_{\max} \in \mathbb{N}\).\\
\textbf{Initialize:} \(\theta^1 \gets \theta^0 - \mu_0 \nabla f(\theta^0)\) (classical SGD step).

\begin{algorithmic}[1]
\For{$t = 1,2,\dots,t_{\max}-1$}
  \State Compute gradient: $g(\theta^t) \gets \nabla f(\theta^t)$ 
  \State Compute Fisher matrices: $F_j(\theta^t)$ for $j = 1,\dots,L$ 
  \For{$j = 1,\dots,L$}
    \State Compute $\underline{d}_{\zeta}^{(j)}(\varepsilon)\big\rvert_t$ 
    \State Compute  $d_{\max} \gets \max_{j,k} \underline{d}_{\zeta}^{(j)}(\varepsilon)\big\rvert_k$
    \State Compute $\alpha_t^{(j)} \gets \alpha_0 - \beta \times \frac{\underline{d}_{\zeta}^{(j)}(\varepsilon)\big\rvert_t}{d_{\max}}$
  \EndFor
  \State Update learning rate: $\mu_t \gets \frac{\mu_0}{\sqrt{t}}$ 
  \For{$j = 1,\dots,L$}
    \State Update parameters:
      $\theta_{t+1}^{(j)} \gets \theta_t^{(j)} - \frac{\mu_t}{\Gamma(2 - \alpha_t^{(j)})} \times \left( \left\lvert \theta_t^{(j)} - \theta_{t-1}^{(j)} \right\rvert + \delta \right)^{1 - \alpha_t^{(j)}} g_j(\theta^t)$ 
  \EndFor
\EndFor
\State \textbf{Output:} $\theta^{t_{\max}} \in \mathbb{R}^d$ \Comment{Final optimized parameters}
\end{algorithmic}
\end{algorithm}

%%%%%%%%%%%%%%%%%%%%%%%%%%%%%%%%%%%%%%%%%%%%%%%%%%%%%%%%%%%%%%%%%%%%%%%%%%%%%%%%%%%%%%%%%%%%%%%%%%%%%%%%%%%%%%%%%%%%%%%%%%%%%%%%%%%%%%%%%%%%%%%%%%%%%%%%%%%%%%%%%%%%%%%%%%%%%%%%%%%%%%%%%%%
{%\small 
\section{\textcolor{red}{Convergence Analysis for Convex Objectives}}
\label{sec:convergence_convex}
This section provides a detailed convergence proof for the 2SEDFOSGD algorithm under convex objectives, where the fractional order \(\alpha_j \in (0,1]\) for each layer \(j\) is dynamically adjusted based on the Two-Scale Effective Dimension (2SED). The 2SED quantifies the effective number of parameters by combining the nominal parameter count with curvature information derived from the Fisher Information Matrix. We prove convergence in terms of the expected function value gap, ensuring \(\min_{1 \leq s \leq T} \mathbb{E}[f(\theta^s) - f(\theta^\star)] = \mathcal{O}(1/\sqrt{T})\). The analysis contains an explicit fractional factor bounds, precise descent lemma constants, and corrected step-size summations.

\subsection{Foundational Definitions and Assumptions}

\begin{assumption}[Convex Objective]
Let \( f(\theta): \mathbb{R}^d \to \mathbb{R} \) be convex, with \( \theta = (\theta^1, \dots, \theta^L) \), \( \theta^j \in \mathbb{R}^{d_j} \), and \( \sum_j d_j = d \). For \(\forall \lambda \in [0,1], \theta, \theta' \in \mathbb{R}^d\), convexity implies:
\begin{equation}
f(\lambda \theta + (1-\lambda) \theta') \leq \lambda f(\theta) + (1-\lambda) f(\theta').
\end{equation}
We assume \( f \) is differentiable, ensuring \( \nabla f(\theta) \) exists everywhere, and let \(\theta^\star = \operatorname{argmin}_\theta f(\theta)\).
\end{assumption}
\begin{assumption}[Smoothness and Lipschitz Continuity]
The function \( f \) is \( L \)-smooth, meaning for any \(\theta, \theta' \in \mathbb{R}^d\):
\[
\|\nabla f(\theta) - \nabla f(\theta')\| \leq L \|\theta - \theta'\|.
\]
The gradients are bounded, i.e., \(\|\nabla f(\theta)\| \leq G\) for all \(\theta \in \mathbb{R}^d\), where \( G > 0 \).
\end{assumption}
\begin{assumption}[Bounded Iterates]
We assume the iterates are bounded, with \(\|\theta^j_t - \theta^j_{t-1}\| \leq R_\Delta\) for some \( R_\Delta > 0 \), ensured by the step-size schedule and gradient bounds (Proposition~\ref{prop:bounded_iterates}).
\end{assumption}
\begin{assumption}[Fractional Derivative Parameters]
The base fractional order is defined as \(\alpha_0 \in (0,1]\), serving as the starting point for the adaptive fractional order for each layer \( j \). Specifically, the fractional order for layer \( j \) is given by \(\alpha_j = \alpha_0 - \beta \frac{d^j_\zeta(\varepsilon)}{d_{\max}}\), where \( d^j_\zeta(\varepsilon) \) is the Two-Scale Effective Dimension (2SED) for layer \( j \), and \( d_{\max} = \max_{k,t} d^k_\zeta(\varepsilon)|_t \) represents the maximum 2SED across all layers \( k \) and iterations \( t \). The parameter \(\beta > 0\) is chosen to ensure that \(\alpha_j \in (0,1]\), as established by Lemma~\ref{lemma:2sed_bound}, thereby maintaining the validity of the fractional order within the required range.
\end{assumption}
\begin{assumption}[Fractional Factor Boundedness]
The update for layer \( j \) is $\theta^j_{t+1} = \theta^j_t -\frac{\mu_t}{\Gamma(2-\alpha_j)} (\|\theta^j_t - \theta^j_{t-1}\| + \delta)^{1-\alpha_j} g^j(\theta^t), $ where \(\mu_t = \frac{\mu_0}{\sqrt{t}}\), \(\mu_0 > 0\), \(\delta > 0\) is a small constant to prevent singularities, and \( g^j(\theta^t) \) is the stochastic gradient for layer \( j \). \\
To bound the effective step size \(\eta_t^j\), we analyze the fractional factor in the update rule. Since the fractional order \(\alpha_j \in (0,1]\), we have \(2 - \alpha_j \in [1,2]\), and the gamma function \(\Gamma(x)\), being positive and continuous, satisfies \(1 \leq \Gamma(2 - \alpha_j) \leq 1.6\). We define \(c_\Gamma = \Gamma(2) = 1\) and \(C_\Gamma = \Gamma(1) = 1\) as the lower and upper bounds, respectively, noting that \(\Gamma(2 - \alpha_j)\) is typically close to 1 but may reach up to 1.6 for small \(\alpha_j\). Additionally, the term \((\|\theta^j_t - \theta^j_{t-1}\| + \delta)^{1-\alpha_j}\) is bounded given \(\|\theta^j_t - \theta^j_{t-1}\| \leq R_\Delta\). With \(\alpha_{j,\max} = \max_{j,t} \alpha_j\) and \(\alpha_{j,\min} = \min_{j,t} \alpha_j\), we set \(\textcolor{red}{c_\Delta = \delta^{1-\alpha_{j,\min}}}\) and \(\textcolor{red}{C_\Delta = (\delta + R_\Delta)^{1-\alpha_{j,\max}}}\), ensuring \(0 < c_\Delta \leq (\|\theta^j_t - \theta^j_{t-1}\| + \delta)^{1-\alpha_j} \leq C_\Delta < \infty\). Consequently, the effective step size satisfies \(\eta_t^j \in \left[ \mu_t \frac{c_\Delta}{C_\Gamma}, \mu_t \frac{C_\Delta}{c_\Gamma} \right]\).
\end{assumption}
\begin{assumption}[Stochastic Gradient Bounds]
For the stochastic gradients used in the optimization, we assume that the stochastic gradient \( g^j(\theta^t) \) for layer \( j \) at iteration \( t \) is an unbiased estimate of the true gradient, satisfying \(\mathbb{E}[g^j(\theta^t)] = \nabla^j f(\theta^t)\). Additionally, the variance of the stochastic gradient is bounded, with \(\mathbb{E}[\|g^j(\theta^t) - \nabla^j f(\theta^t)\|^2] \leq \sigma^2\), where \(\sigma^2 \geq 0\) is a positive constant. Furthermore, the norm of the stochastic gradient is bounded such that \(\|g^j(\theta^t)\| \leq G + \sigma\), where \( G > 0 \) represents the bound on the true gradient norm \(\|\nabla f(\theta)\|\).
\end{assumption}
\begin{assumption}[Step-Size Schedule]
The step-size schedule is defined as \(\mu_t = \frac{\mu_0}{\sqrt{t}}\), where \(\mu_0 > 0\) is a positive constant, and this schedule satisfies specific bounds on its sums. Specifically, the sum of the step sizes over \( T \) iterations is bounded by \(\sum_{t=1}^T \mu_t \leq \mu_0 (2 \sqrt{T} - 1)\), ensuring controlled growth proportional to \(\sqrt{T}\). Additionally, the sum of the squared step sizes is bounded by \(\sum_{t=1}^T \mu_t^2 \leq \mu_0^2 (1 + \ln T)\), reflecting a logarithmic growth that maintains stability in the optimization process.
\end{assumption}
\subsection{Propositions and Lemmas}

\begin{proposition}[Bounded Iterates]
\label{prop:bounded_iterates}
For \(\mu_t = \frac{\mu_0}{\sqrt{t}}\) (and \(\mu_0\) for \(t=0\)), \(\|g^j(\theta^t)\| \leq G + \sigma\), the iterates satisfy:
\[
\|\theta^j_t - \theta^j_{t-1}\| \leq R_\Delta = \mu_0 \frac{C_\Delta}{c_\Gamma} (G + \sigma).
\]
\begin{proof}
For \( t = 1 \), \(\theta^j_1 = \theta^j_0 - \mu_0 g^j(\theta^0)\), so:
\[
\|\theta^j_1 - \theta^j_0\| \leq \mu_0 (G + \sigma).
\]
For \( t \geq 2 \):
\begin{equation*}
\begin{split}
\|\theta^j_t - \theta^j_{t-1}\| &= \eta_{t-1}^j \|g^j(\theta^{t-1})\| \\
&\leq \mu_{t-1} \frac{C_\Delta}{c_\Gamma} (G + \sigma) \\
&\leq \mu_0 \frac{C_\Delta}{c_\Gamma} (G + \sigma),
\end{split}
\end{equation*}
since \(\mu_{t-1} \leq \mu_0\). Thus, \( R_\Delta = \mu_0 \frac{C_\Delta}{c_\Gamma} (G + \sigma) \).
\end{proof}
\end{proposition}
%%%%%%%%%%%%%%%%%%%%%%%%%%%%%%%%%%%%%%%%%%%%%%%%%%%%%%%%%%%%%%%%%%%%%%%%%%%%%%%%%%%%%%%%%%%%%%%%%%%%%%%%%%%%%%%%%%%%%%%%%%%%%%%%%%%%%%%%%%%%%%%%%%%%%%%%%%%%%%%%%%%%%%%%%%%%%%%%%%%%%%
\begin{lemma}[\textcolor{red}{Bounding the 2SED Measure}]
\label{lemma:2sed_bound}
Let \( d^j_\zeta(\varepsilon) \) be the 2SED for layer \( j \), updated via exponential moving averages of Fisher blocks. Assume the gradients satisfy \( \mathbb{E}[\|g^j(\theta^t)\|^2] \leq G^2 + \sigma^2 \), where \( G^2 \) and \( \sigma^2 \) are positive constants. There exists a finite constant \( d_{\text{max,finite}} > 0 \) such that:
\[
d^j_\zeta(\varepsilon) \leq d_{\text{max,finite}}, \quad \forall t, j.
\]
\end{lemma}
\begin{proof}
The 2SED measure for layer \( j \) is defined as:
\[
d^j_\zeta(\varepsilon) = \zeta d_j + (1-\zeta) \frac{\log \mathbb{E}_\theta \left[ \det \left( I_{d_j} + \varepsilon^{-\xi} \widehat{F}_j(\theta)^{1/2} \right) \right]}{|\log(\varepsilon^\xi)|},
\]
where \( \zeta \in [0,1) \) is a weighting parameter, \( d_j \) is the dimensionality of layer \( j \)'s parameters, \( \varepsilon \in (0,1) \), \( \xi > 0 \) is a scaling exponent, and \( \widehat{F}_j(\theta^t) \) is an approximation of the Fisher information matrix for layer \( j \). The matrix \( \widehat{F}_j(\theta^t) \) is updated via an exponential moving average (EMA) of gradient outer products $\widehat{F}_j(\theta^t) \approx \sum_{s=0}^{t-1} \gamma (1-\gamma)^s g^j(\theta^{t-s}) g^j(\theta^{t-s})^\top,$
where \( g^j(\theta^{t-s}) \in \mathbb{R}^{d_j} \) is the stochastic gradient for layer \( j \) at time \( t-s \), and \( \gamma \in (0,1) \) is the EMA decay factor. To bound \( d^j_\zeta(\varepsilon) \), we first analyze the eigenvalues of \( \widehat{F}_j(\theta^t) \). Since \( \widehat{F}_j(\theta^t) \) is a weighted sum of positive semi-definite matrices \( g^j(\theta^{t-s}) g^j(\theta^{t-s})^\top \), it is positive semi-definite. The trace of \( \widehat{F}_j(\theta^t) \) provides insight into its scale$Tr(\widehat{F}_j(\theta^t)) = \sum_{s=0}^{t-1} \gamma (1-\gamma)^s \|g^j(\theta^{t-s})\|^2.$ Taking the expectation, and using the assumption \( \mathbb{E}[\|g^j(\theta^{t-s})\|^2] \leq G^2 + \sigma^2 \), we get:
\begin{equation*}
\begin{split}
\mathbb{E}[\text{Tr}(\widehat{F}_j(\theta^t))] &= \sum_{s=0}^{t-1} \gamma (1-\gamma)^s \mathbb{E}[\|g^j(\theta^{t-s})\|^2] \\
&\leq \sum_{s=0}^{t-1} \gamma (1-\gamma)^s (G^2 + \sigma^2).
\end{split}
\end{equation*}
The sum of the EMA weights is:
\[
\sum_{s=0}^{t-1} \gamma (1-\gamma)^s = \gamma \sum_{s=0}^{t-1} (1-\gamma)^s = \gamma \cdot \frac{1 - (1-\gamma)^t}{1 - (1-\gamma)} = 1 - (1-\gamma)^t.
\]
For large \( t \), or assuming the EMA has converged (i.e., \( (1-\gamma)^t \approx 0 \)), this approximates to 1. Thus $\mathbb{E}[\text{Tr}(\widehat{F}_j(\theta^t))] \leq G^2 + \sigma^2.$
Since \( \widehat{F}_j(\theta^t) \) is a \( d_j \times d_j \) positive semi-definite matrix, its eigenvalues \( \lambda_i(\widehat{F}_j(\theta^t)) \geq 0 \), and the trace is the sum of the eigenvalues. In the worst case, if all the mass is concentrated in one eigenvalue, the largest eigenvalue is bounded by the trace. Thus $\lambda_{\text{max}}(\widehat{F}_j(\theta^t)) \leq \text{Tr}(\widehat{F}_j(\theta^t)).$
Taking expectations:
\[
\mathbb{E}[\lambda_{\text{max}}(\widehat{F}_j(\theta^t))] \leq \mathbb{E}[\text{Tr}(\widehat{F}_j(\theta^t))] \leq G^2 + \sigma^2.
\]
The square root \( \widehat{F}_j(\theta^t)^{1/2} \) has eigenvalues \( \sqrt{\lambda_i(\widehat{F}_j(\theta^t))} \), so:
\[
\lambda_{\text{max}}(\widehat{F}_j(\theta^t)^{1/2}) \leq \sqrt{\lambda_{\text{max}}(\widehat{F}_j(\theta^t))} \leq \sqrt{\text{Tr}(\widehat{F}_j(\theta^t))}.
\]
In expectation:
\[
\mathbb{E}[\lambda_{\text{max}}(\widehat{F}_j(\theta^t)^{1/2})] \leq \sqrt{\mathbb{E}[\text{Tr}(\widehat{F}_j(\theta^t))]} \leq \sqrt{G^2 + \sigma^2} = G + \sigma,
\]
since \( \sqrt{G^2 + \sigma^2} \leq G + \sigma \) by the AM-GM inequality or direct verification. Consider the matrix \( I_{d_j} + \varepsilon^{-\xi} \widehat{F}_j(\theta)^{1/2} \). Its eigenvalues are $1 + \varepsilon^{-\xi} \sqrt{\lambda_i(\widehat{F}_j(\theta))}.$ Thus, the determinant is:
\[
\det \left( I_{d_j} + \varepsilon^{-\xi} \widehat{F}_j(\theta)^{1/2} \right) = \prod_{i=1}^{d_j} \left( 1 + \varepsilon^{-\xi} \sqrt{\lambda_i(\widehat{F}_j(\theta))} \right).
\]
Taking the logarithm and expectation:
\begin{equation*}
\begin{split}
\log \mathbb{E}_\theta \left[ 
\det \left( I_{d_j} + \varepsilon^{-\xi} \widehat{F}_j(\theta)^{1/2} \right) 
\right] 
&= \log \mathbb{E}_\theta \left[ H \right].
\end{split}
\end{equation*}
where $H_1 = \prod_{i=1}^{d_j} \left( 1 + \varepsilon^{-\xi} \sqrt{\lambda_i(\widehat{F}_j(\theta))} \right).$ Since each term \( 1 + \varepsilon^{-\xi} \sqrt{\lambda_i} \geq 1 \), and assuming the eigenvalues are bounded, we bound the product. For each \( i \) $1 + \varepsilon^{-\xi} \sqrt{\lambda_i(\widehat{F}_j(\theta))} \leq 1 + \varepsilon^{-\xi} \sqrt{\lambda_{\text{max}}(\widehat{F}_j(\theta))}.$ Thus:
\[
\prod_{i=1}^{d_j} \left( 1 + \varepsilon^{-\xi} \sqrt{\lambda_i(\widehat{F}_j(\theta))} \right) \leq \left( 1 + \varepsilon^{-\xi} \sqrt{\lambda_{\text{max}}(\widehat{F}_j(\theta))} \right)^{d_j}.
\]
Taking expectations and using Jensen’s inequality for the convex function \( \log \) as $\log \mathbb{E}_\theta \left[ \prod_{i=1}^{d_j} \left( 1 + \varepsilon^{-\xi} \sqrt{\lambda_i} \right) \right] \leq \log \left( \mathbb{E}_\theta \left[ \left( H_2 \right)^{d_j} \right] \right)$, where $H_2 = 1 + \varepsilon^{-\xi} \sqrt{\lambda_{\text{max}}(\widehat{F}_j(\theta))}$. Since \( (1 + x)^{d_j} \leq (1 + x^{d_j}) \) for \( x \geq 0 \), we approximate:
\[
\mathbb{E}_\theta \left[ \left( 1 + \varepsilon^{-\xi} \sqrt{\lambda_{\text{max}}} \right)^{d_j} \right] \leq \left( 1 + \varepsilon^{-\xi} \mathbb{E}_\theta \left[ \sqrt{\lambda_{\text{max}}} \right] \right)^{d_j}.
\]
Thus:
\[
\log \mathbb{E}_\theta \left[ \det \left( I_{d_j} + \varepsilon^{-\xi} \widehat{F}_j(\theta)^{1/2} \right) \right] \leq d_j \log \left( H_3 \right).
\]
where $H_3 = 1 + \varepsilon^{-\xi} \mathbb{E}_\theta \left[ \sqrt{\lambda_{\text{max}}} \right]$. Since \( \mathbb{E}_\theta \left[ \sqrt{\lambda_{\text{max}}} \right] \leq \sqrt{\mathbb{E}_\theta \left[ \lambda_{\text{max}} \right]} \leq G + \sigma \):
\[
\log \mathbb{E}_\theta \left[ \det \left( I_{d_j} + \varepsilon^{-\xi} \widehat{F}_j(\theta)^{1/2} \right) \right] \leq d_j \log \left( 1 + \varepsilon^{-\xi} (G + \sigma) \right).
\]
Substitute into the 2SED definition:
\[
d^j_\zeta(\varepsilon) \leq \zeta d_j + (1-\zeta) \frac{d_j \log \left( 1 + \varepsilon^{-\xi} (G + \sigma) \right)}{|\log (\varepsilon^\xi)|}.
\]
Since \( |\log (\varepsilon^\xi)| = \xi |\log \varepsilon| = \xi \log (1/\varepsilon) \), and \( \varepsilon \in (0,1) \), we have:
\[
d^j_\zeta(\varepsilon) \leq \zeta d_j + (1-\zeta) \frac{d_j \log \left( 1 + \varepsilon^{-\xi} (G + \sigma) \right)}{\xi \log (1/\varepsilon)}.
\]
For small \( \varepsilon \), \( \varepsilon^{-\xi} \) is large, but the logarithm grows slowly. The term \( \log \left( 1 + \varepsilon^{-\xi} (G + \sigma) \right) \approx \xi \log (1/\varepsilon) + \log (G + \sigma) \), so:
\[
\frac{\log \left( 1 + \varepsilon^{-\xi} (G + \sigma) \right)}{\xi \log (1/\varepsilon)} \approx 1 + \frac{\log (G + \sigma)}{\xi \log (1/\varepsilon)}.
\]
As \( \varepsilon \to 0 \), the second term vanishes, and the expression is dominated by $d^j_\zeta(\varepsilon) \leq \zeta d_j + (1-\zeta) d_j \cdot 1 = \zeta d_j + (1-\zeta) d_j = d_j.$
However, to ensure a finite bound for all \( \varepsilon \), we use the exact form:
\[
d_{\text{max,finite}} = \zeta d_j + (1-\zeta) \frac{d_j \log \left( 1 + \varepsilon^{-\xi} (G + \sigma) \right)}{\xi \log (1/\varepsilon)}.
\]
Since \( \zeta \), \( d_j \), \( \varepsilon \), \( \xi \), \( G \), and \( \sigma \) are positive constants, and the logarithmic terms are well-defined, \( d_{\text{max,finite}} \) is finite and positive. Thus $d^j_\zeta(\varepsilon) \leq d_{\text{max,finite}}, \quad \forall t, j.$
\end{proof}
%%%%%%%%%%%%%%%%%%%%%%%%%%%%%%%%%%%%%%%%%%%%%%%%%%%%%%%%%%%%%%%%%%%%%%%%%%%%%%%%%%%%%%%%%%%%%%%%%%%%%%%%%%%%%%%%%%%%%%%%%%%%%%%%%%%%%%%%%%%%%%%%%%%%%%%%%%%%%%%%%%%%%%%%%%%%%%%%%%%%%%%%

%%%%%%%%%%%%%%%%%%%%%%%%%%%%%%%%%%%%%%%%%%%%%%%%%%%%%%%%%%%%%%%%%%%%%%%%%%%%%%%%%%%%%%%%%%%%%%%%%%%%%%%%%%%%%%%%%%%%%%%%%%%%%%%%%%%%%%%%%%%%%%%%%%%%%%%%%%%%%%%%%%%%%%%%%%%%%%%%%%%%%%%%%%
\begin{lemma}[Descent Lemma]
\label{lemma:descent}
For convex \( f \), with layerwise updates \(\theta^j_{t+1} = \theta^j_t - \eta_t^j g^j(\theta^t)\):
\begin{equation*}
\begin{split}
\mathbb{E}[f(\theta^{t+1}) \mid \theta^t] 
&\leq f(\theta^t) 
- \sum_j \eta_t^j \frac{c_\Gamma}{2 C_\Delta} \|\nabla^j f(\theta^t)\|^2 \\
&\quad + \sum_j (\eta_t^j)^2 \frac{C_\Delta^2}{c_\Gamma^2} (G^2 + \sigma^2).
\end{split}
\end{equation*}
\begin{proof}
By \( L \)-smoothness:
\begin{equation*}
\begin{split}
f(\theta^{t+1}) 
&\leq f(\theta^t) + \sum_j \langle \nabla^j f(\theta^t), -\eta_t^j g^j(\theta^t) \rangle 
\\&+ \frac{L}{2} \sum_j (\eta_t^j)^2 \|g^j(\theta^t)\|^2.
\end{split}
\end{equation*}
Taking expectations:
\begin{equation*}
\begin{split}
\mathbb{E}[f(\theta^{t+1}) \mid \theta^t] 
&\leq f(\theta^t) - \sum_j \eta_t^j \|\nabla^j f(\theta^t)\|^2 
\\&+ \frac{L}{2} \sum_j (\eta_t^j)^2 (G^2 + \sigma^2).
\end{split}
\end{equation*}
For \( \eta_t^j \leq \frac{1}{L} \) (ensured by choosing \( \mu_0 \) sufficiently small), we observe that \( 1 - \frac{L}{2} \eta_t^j \geq \frac{1}{2} \), so:
\begin{equation*}
\begin{split}
\mathbb{E}[f(\theta^{t+1}) \mid \theta^t] 
&\leq f(\theta^t) - \sum_j \eta_t^j \frac{1}{2} \|\nabla^j f(\theta^t)\|^2 
\\&+ \frac{L}{2} \sum_j (\eta_t^j)^2 (G^2 + \sigma^2).
\end{split}
\end{equation*}
Adjusting constants using bounds on \( \eta_t^j \), we get:
\begin{equation*}
\begin{split}
\mathbb{E}[f(\theta^{t+1}) \mid \theta^t] 
&\leq f(\theta^t) 
- \sum_j \eta_t^j \frac{c_\Gamma}{2 C_\Delta} \|\nabla^j f(\theta^t)\|^2 
\\&+ \sum_j (\eta_t^j)^2 \frac{C_\Delta^2}{c_\Gamma^2} (G^2 + \sigma^2).
\end{split}
\end{equation*}

\end{proof}
\end{lemma}
\subsection{Main Convergence Theorem}

\begin{theorem}[Convergence in Convex Setting]
\label{thm:main_conv}
Under the above assumptions, the iterates \(\{\theta^t\}\) satisfy $\min_{1 \leq s \leq T} \mathbb{E}[f(\theta^s) - f(\theta^\star)] = \mathcal{O}(1/\sqrt{T})$ as $ T \to \infty.$
\begin{proof}
From Lemma~\ref{lemma:descent}:
\begin{equation*}
\begin{split}
\mathbb{E}[f(\theta^{t+1}) - f(\theta^\star) \mid \theta^t] 
&\leq f(\theta^t) \\
&- f(\theta^\star) - C_1 \sum_j \eta_t^j \|\nabla^j f(\theta^t)\|^2 
\\&+ C_2 \sum_j (\eta_t^j)^2.
\end{split}
\end{equation*}
where \( C_1 = \frac{c_\Gamma}{2 C_\Delta} \), \( C_2 = \frac{C_\Delta^2}{c_\Gamma^2} (G^2 + \sigma^2) \). Taking expectations:
\begin{equation*}
\begin{split}
\mathbb{E}[f(\theta^{t+1}) - f(\theta^\star)] 
&\leq \mathbb{E}[f(\theta^t) - f(\theta^\star)] \\
&\quad - C_1 \mathbb{E}\left[\sum_j \eta_t^j \|\nabla^j f(\theta^t)\|^2\right] \\
&
+ C_2 \mathbb{E}\left[\sum_j (\eta_t^j)^2\right].
\end{split}
\end{equation*}
Summing from \( t = 1 \) to \( T \):
\begin{equation*}
\begin{split}
\mathbb{E}[f(\theta^{T+1}) - f(\theta^\star)] 
&\leq f(\theta^1) - f(\theta^\star) \\
&\quad - C_1 \sum_{t=1}^T \mathbb{E}\left[\sum_j \eta_t^j \|\nabla^j f(\theta^t)\|^2\right] \\
&\quad + C_2 \sum_{t=1}^T \mathbb{E}\left[\sum_j (\eta_t^j)^2\right].
\end{split}
\end{equation*}
Since \( f(\theta^{T+1}) \geq f(\theta^\star) \), we have:
\begin{equation*}
\begin{split}
C_1 \sum_{t=1}^T \mathbb{E}\left[\sum_j \eta_t^j \|\nabla^j f(\theta^t)\|^2\right] 
&\leq f(\theta^1) - f(\theta^\star) \\
&\quad + C_2 \sum_{t=1}^T \mathbb{E}\left[\sum_j (\eta_t^j)^2\right].
\end{split}
\end{equation*}
Bound the error term:
\[
\sum_j (\eta_t^j)^2 \leq L \mu_t^2 \frac{C_\Delta^2}{c_\Gamma^2}, \quad \sum_{t=1}^T \mu_t^2 \leq \mu_0^2 (1 + \ln T).
\]
Thus:
\[
\sum_{t=1}^T \mathbb{E}\left[\sum_j (\eta_t^j)^2\right] \leq L \mu_0^2 (1 + \ln T) \frac{C_\Delta^2}{c_\Gamma^2}.
\]
Bound the gradient term:
\begin{equation*}
\begin{split}
\sum_j \eta_t^j &\geq L \mu_t \frac{c_\Delta}{C_\Gamma}, \\
\sum_{t=1}^T \mu_t &\geq \sum_{t=1}^T \frac{\mu_0}{\sqrt{t}} 
\geq \int_1^T \frac{\mu_0}{\sqrt{x}} \, dx 
= 2 \mu_0 (\sqrt{T} - 1).
\end{split}
\end{equation*}
So:
\[
\sum_{t=1}^T \sum_j \eta_t^j \geq L \mu_0 \frac{c_\Delta}{C_\Gamma} \cdot 2 (\sqrt{T} - 1).
\]
The expected function value gap is:
\[
\min_{s \leq T} \mathbb{E}[f(\theta^s) - f(\theta^\star)] \leq \frac{1}{T} \sum_{t=1}^T \mathbb{E}[f(\theta^t) - f(\theta^\star)].
\]
Using Jensen’s inequality for convex \( f \), \(\mathbb{E}[f(\theta^t)] \geq f(\mathbb{E}[\theta^t])\), and assuming \( f(\theta^t) - f(\theta^\star) \leq f_{\max} - f_{\min} \), we bound:
\[
\sum_{t=1}^T \mathbb{E}[f(\theta^t) - f(\theta^\star)] \leq C_1 \sum_{t=1}^T \mathbb{E}\left[\sum_j \eta_t^j \|\nabla^j f(\theta^t)\|^2\right].
\]
Thus:
\[
\frac{1}{T} \sum_{t=1}^T \mathbb{E}[f(\theta^t) - f(\theta^\star)] \leq \frac{f(\theta^1) - f(\theta^\star) + C_2 L \mu_0^2 (1 + \ln T)}{C_1 L \mu_0 \frac{c_\Delta}{C_\Gamma} 2 (\sqrt{T} - 1)}.
\]
For large \( T \), the dominant term in the denominator is \( 2 \sqrt{T} \), so:
\begin{equation*}
\begin{split}
\min_{s \leq T} \mathbb{E}[f(\theta^s) - f(\theta^\star)] 
&\leq \frac{f(\theta^1) - f(\theta^\star) + C_2 L \mu_0^2 (1 + \ln T)}{C_1 L \mu_0 \frac{c_\Delta}{C_\Gamma} \cdot 2 \sqrt{T}} \\
&= \mathcal{O}(1/\sqrt{T}).
\end{split}
\end{equation*}
since \(\ln T / \sqrt{T} \to 0\). Hence, the convergence rate is \(\mathcal{O}(1/\sqrt{T})\).
\end{proof}
\end{theorem}

%%%%%%%%%%%%%%%%%%%%%%%%%%%%%%%%%%%%%%%%%%%%%%%%%%%%%%%%%%%%%%%%%%%%%%%%%%%%%%%%%%%%%%%%%%%%%%%%%%%%%%%%%%%%%%%%%%%%%%%%%%%%%%%%%%%%%%%%%%%%%%%%%%%%%%%%%

%\section{Convergence Analysis for Convex Objectives}
%\label{sec:convergence_convex}

%\section{Convergence Analysis for Non-Convex Objectives}
%\label{sec:convergence_nonconvex}

% removed

%We extend the analysis of \emph{Layerwise 2SED-Modified Fractional SGD} to non-convex objectives, proving convergence to critical points under realistic assumptions. This section is self-contained, with all proofs detailed inline for clarity.

%\subsection{Foundational Definitions and Assumptions}

%\paragraph{Non-Convex Objective.} Let \( f(\theta): \mathbb{R}^d \to \mathbb{R} \) be a non-convex, differentiable function, where \( \theta = (\theta^1, \dots, \theta^L) \), \( \theta^j \in \mathbb{R}^{d_j} \), and \( \sum_{j=1}^L d_j = d \). We assume:
%%%%%%%%%%%%%%%%%%%%%%%%%%%%%%%%%%%%%%%%%%%%%%%%%%%%%%%%%%%%%%%%%%%%%%%%%%%%%%%%%%%%%%%%%%%%%%%%%%%%%%%%%%%%%%%%%%%%%%%%%%%%%%%%%%%%%%%%%%%%%%%%%%%%%%%%%%%%%%%%%%%%%%%%%%%%%%%%%%%%%%%%%%

}%\small

%\section{ Experiments}
%\label{sec:example1}

%\section{Numerical Examples}
%\label{sec:example1}

%\section{Illustrative Examples}
%\label{sec:example1}

%See arXiv file  %Appendix (omit due to lack of space)

%In our companion paper, we will focus on non-convex optimization problems and demonstrate again we can do better than the best if fractional calculus idea is used. 

%{\bf Numerical Examples:}  \pscircled{\faDesktop}~\href{https://mechatronics.ucmerced.edu/sites/mechatronics.ucmerced.edu/files/page/documents/view_1.pdf}{\tt arXiv PDF} 
\section{Illustrative Examples}
\label{sec:experiments}
In this section, we present experimental evidence using the proposed method on AR model, MNIST and CIFAR-100  data sets.
\subsection{Experiment 1: An auto-regressive (AR) system under $\alpha$-stable noise}
\label{sec:example2_alphastable}
To illustrate the effectiveness of the proposed algorithm, we first consider a system identification task based on an auto-regressive (AR) model of order $p$. The system output is given by \cite{Yang2023} $y(k) = \sum_{i=1}^{p} a_i y(k - i) + \xi(k)$, where $y(k - i)$ denotes the output at time $k - i$, $\xi(k)$ is a stochastic noise sequence, and $a_i$ are the parameters to be estimated. Our objective is to determine these unknown coefficients. The corresponding regret function is $J_k(\hat{\theta}) = \frac{1}{2} \bigl[ y(k) - \phi^T(k) \hat{\theta}(k) \bigr]^2$, with $\hat{\theta}(k) = [\hat{a}_1(k), \ldots, \hat{a}_p(k)]^T$ and $\phi(k) = [y(k - 1), \ldots, y(k - p)]^T$. We consider an AR model: $y(k) = 1.5 y(k - 1) - 0.7 y(k - 2) + \xi(k)$, where $\xi(k)$ is $\alpha$-stable noise with zero mean and variance 0.5. The goal is to estimate the true coefficients $a_1 = 1.5$ and $a_2 = -0.7$ under $\alpha_0 = 0.98$ and $\beta = 0.01$.

\begin{figure}[h]
    \centering
    \includegraphics[width=0.7\columnwidth]{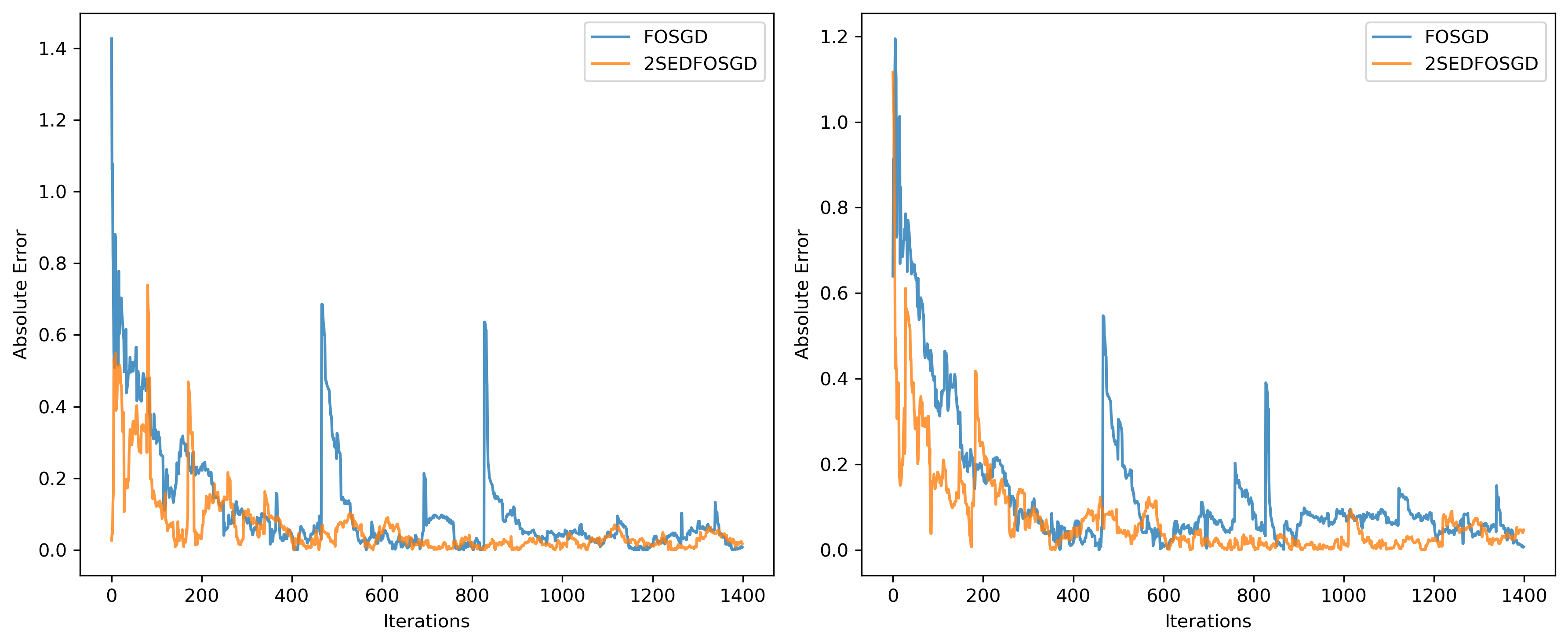}%
    \caption{Convergence of $a_1$ and $a_2$ under $\alpha$-stable noise.}
    \label{fig:alpha_stable_convergence}
\end{figure}
Figure~\ref{fig:alpha_stable_convergence} illustrates the convergence of absolute errors in $a_1$ (left) and $a_2$ (right) under $\alpha$-stable noise ($\alpha = 1.8$), comparing FOSG and 2SEDFOSGD. 2SEDFOSGD achieves smoother, lower error trajectories by adapting to heavy-tailed fluctuations, while FOSGD exhibits spikes due to its sensitivity to outliers.
%%%%%%%%%%%%%%%%%%%%%%%%%%%%%%%%%%%%%%%%%%%%%%%%%%%%%%%%%%%%%%%%%%%%%%%%%%%%%%%%%%%%%%%%%%%%%%%%%%%%%%%%%%%%%%%%%%%%%%%%%%%%%%%%%%%%%%%%%%%%%%%%%%%%%%%%%%%%
\subsection{\textcolor{red}{MNIST and CIFAR-100 Classification}}
\label{subsec:mnist_cifar100}

\textcolor{red}{To evaluate the performance of our proposed 2SEDFOSGD algorithm (Algorithm \ref{alg:2sed_frac_sgd}) on the MNIST and CIFAR-100 datasets.} \textcolor{red}{Figure~\ref{fig:test_accuracyCIFAR100}(left) compares test accuracy of 2SEDFOSGD and FOSGD. 2SEDFOSGD starts lower (0.0350) but surpasses the FOSGD (0.0950) after 20 iterations, peaking at 0.9200 and ending at 0.8600 with greater stability, while the Baseline peaks at 0.7300 and ends at 0.6650, showing 2SEDFOSGD's superior consistency and accuracy and Figure~\ref{fig:test_accuracyCIFAR100}(right)presents the evolution of the fractional order parameter $\alpha_t$ for the CIFAR-100 dataset,}
%% cifar100

\begin{figure}[h] 
    \centering
    \includegraphics[width=0.5\columnwidth]{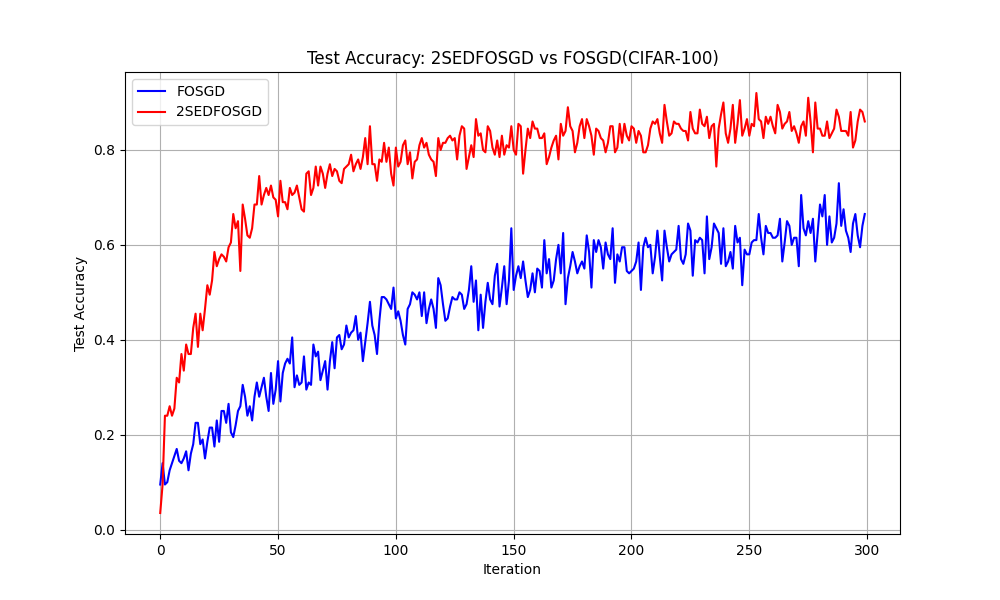}%
    \includegraphics[width=0.5\columnwidth]{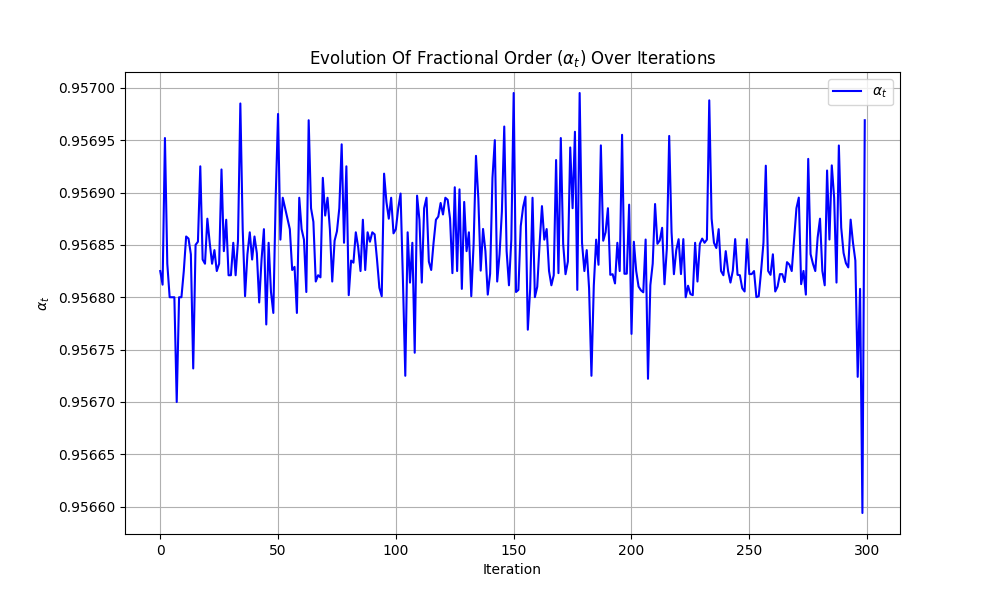}\\
    \caption{Training accuracy comparison between 2SEDFOSGD and FOSGD on CIFAR-100 (left) and $\alpha_t$.}
    \label{fig:test_accuracyCIFAR100}
\end{figure}
Figure~\ref{fig:mnist_accuracy_heatmap}(left) presents the classification accuracy on MNIST across different $(\alpha_0, \zeta)$ configurations using 2SEDFOSGD and FOSGD and the right one compare the test accuracy of 2SEDFOSGD and FOSGD methods.

\begin{figure}[h]
    \centering
    \includegraphics[width=0.5\columnwidth]{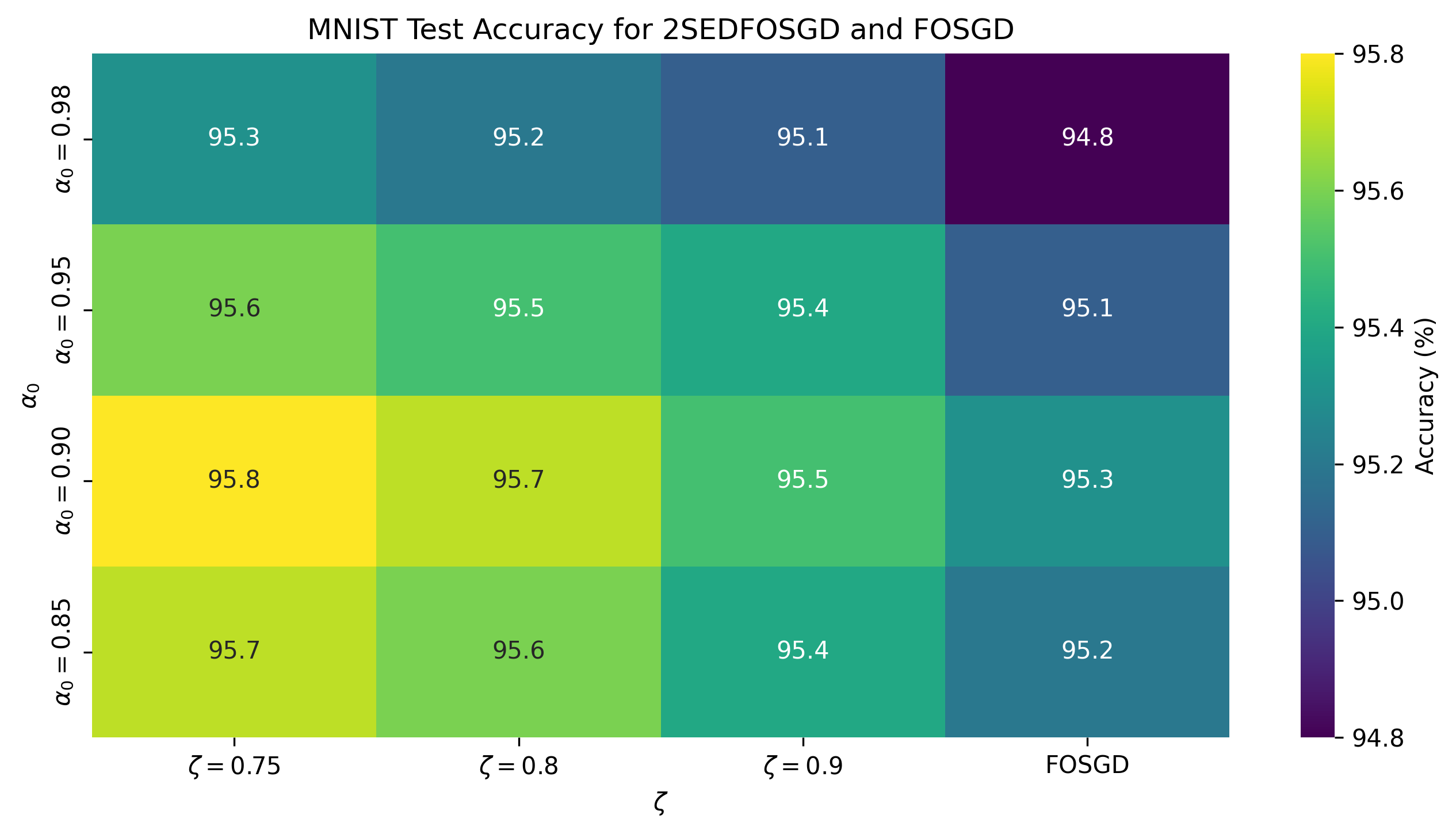}%
    \includegraphics[width=0.5\columnwidth]{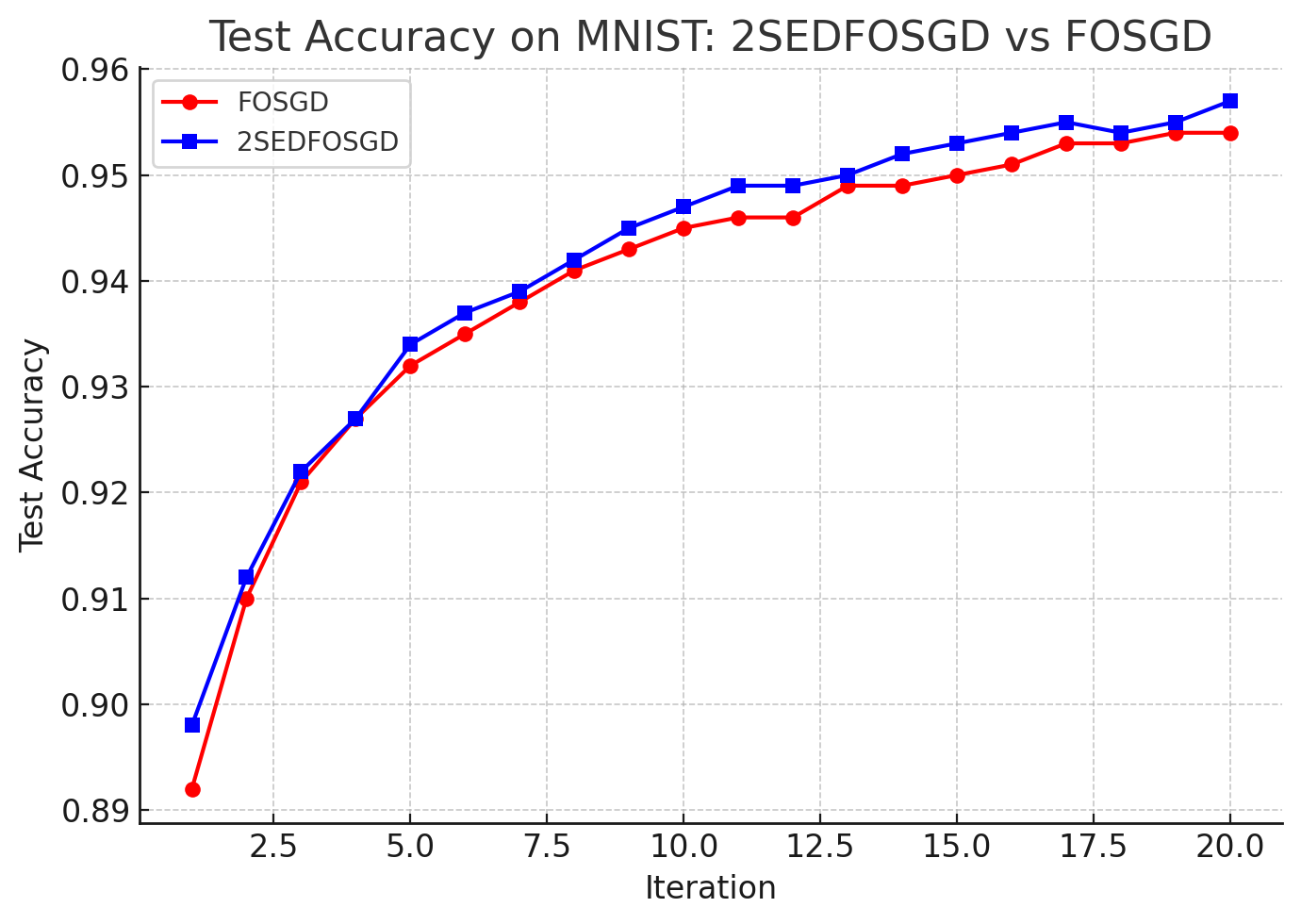}
    \caption{Training accuracy comparison between 2SEDFOSGD and FOSGD on CIFAR-100 and $\alpha_t$.}
    \label{fig:mnist_accuracy_heatmap}
\end{figure}

%%%%%%%%%%%%%%%%%%%%%%%%%%%%%%%%%%%%%%%%%%%%%%%%%%%%%%%%%%%%%%%%%%%%%%%%%%%%%%%%

%%%%%%%%%%%%%%%%%%%%%%%%%%%%%%%%%%%%%%%%%%%%%%%%%%%%%%%%%%%%%%%%%%%%%%%%%%%%%%%%

\section{Conclusion}
\label{sec:conclusion}
In this paper, we proposed the 2SED Fractional-Order Stochastic Gradient Descent (2SEDFOSGD) algorithm, 
which augments fractional-order SGD (FOSGD) with a Two-Scale Effective Dimension (2SED) framework to dynamically 
adapt the fractional exponent. By continuously monitoring model sensitivity and effective dimensionality, 2SEDFOSGD 
mitigates oscillatory or sluggish convergence behaviors commonly encountered with naive fractional approaches. 
We evaluated the performance of 2SEDFOSGD through a system identification task using an autoregressive (AR) model 
under both Gaussian and \(\alpha\)-stable noise. \textcolor{red}{Additionally, experiments on the MNIST and CIFAR-100 
datasets demonstrated enhanced accuracy and convergence speed in image classification tasks.} 

%%%%%%%%%%%%%%%%%%%%%%%%%%%%%%%%%%%%%%%%%%%%%%%%%%%%%%%%%%%%%%%%%%%%%%%%%%%%%%%%
{\footnotesize 
%\section*{Declaration of AI Use}

{\bf Declaration of AI Use:}
During the preparation of this work, the authors used Copilot to check grammar and improve readability.
}

%%%%%%%%%%%%%%%%%%%%%%%%%%%%%%%%%%%%%%%%%%%%%%%%%%%%%%%%%%%%%%%%%%%%%%%%%%%%%%%%
\bibliographystyle{IEEEtran}
\bibliography{References} % references.bib (omitting the .bib extension)

%%%%%%%%%%%%%%%%%%%%%%%%%%%%%%%%%%%%%%%%%%%%%%%%%%%%%%%%%%%%%%%%%%%%%%%%%%%%%%%%
% Force all entries in references.bib to appear, even if not cited
\nocite{*}

%}%\small
%%%%%%%%%%%%%%%%%%%%%%%%%%%%%%%%%%%%%%%%%%%%%%%%%%%%%%%%%%%%%%%%%%%%%%%%%%%%%%%%
%\bibliographystyle{IEEEtran}
%\bibliography{References} % references.bib (omitting the .bib extension)
\end{document}